\documentclass{article} % For LaTeX2e
\usepackage{iclr2024_conference,times}
\usepackage{amsmath,amsfonts,bm}

\def\eqref#1{equation~\ref{#1}}
\def\1{\bm{1}}

\def\mA{{\bm{A}}}

\DeclareMathAlphabet{\mathsfit}{\encodingdefault}{\sfdefault}{m}{sl}
\SetMathAlphabet{\mathsfit}{bold}{\encodingdefault}{\sfdefault}{bx}{n}

\newcommand{\E}{\mathbb{E}}

\usepackage{hyperref}
\usepackage{url}

\usepackage{amsmath,amsfonts,amssymb}
\usepackage{mathtools}
\usepackage{amsthm} % better to load after ams math
\usepackage{latexsym}
\usepackage{relsize}

\newcommand{\xc}[1]{\textsf{\color{magenta} $\{$ Xinyi : #1$\}$}}

\newcommand{\A}{\mathcal{A}}

\newcommand{\K}{\ensuremath{\mathcal K}}

\newcommand{\ignore}[1]{}

\theoremstyle{plain}
\newtheorem{theorem}{Theorem}
\newtheorem{lemma}[theorem]{Lemma}

\newtheorem{assumption}{Assumption}

\newtheorem*{theorem*}{Theorem}
\newtheorem*{lemma*}{Lemma}
\newtheorem*{corollary*}{Corollary}
\newtheorem*{proposition*}{Proposition}
\newtheorem*{claim*}{Claim}
\newtheorem*{fact*}{Fact}
\newtheorem*{observation*}{Observation}
\newtheorem*{assumption*}{Assumption}

\theoremstyle{definition}

\newtheorem*{definition*}{Definition}
\newtheorem*{remark*}{Remark}
\newtheorem*{example*}{Example}

 \theoremstyle{plain}
\newtheorem*{theoremaux}{\theoremauxref}
\gdef\theoremauxref{1}

\DeclareMathAlphabet{\mathbfsf}{\encodingdefault}{\sfdefault}{bx}{n}

\def\mA{{\mathcal A}}

\newcommand{\reals}{\mathbb{R}}

\renewcommand{\leq}{~\le~}

\let\oldtfrac\tfrac
\renewcommand{\tfrac}[2]{\smash{\oldtfrac{#1}{#2}}}

\let\nablaold\nabla
\renewcommand{\nabla}{\nablaold\mkern-2.5mu}

\renewcommand{\epsilon}{\varepsilon}

\renewcommand{\tilde}{\widetilde}
\renewcommand{\hat}{\widehat}
\usepackage{algorithm}
\usepackage{algorithmic}

\usepackage[utf8]{inputenc} % allow utf-8 input
\usepackage[T1]{fontenc}    % use 8-bit T1 fonts

\usepackage{hyperref}
\usepackage{xcolor}
\hypersetup{
    colorlinks,
    linkcolor={red!50!black},
    citecolor={blue!50!black},
    urlcolor={blue!80!black}
}

\usepackage{url}            % simple URL typesetting
\usepackage{booktabs}       % professional-quality tables
\usepackage{amsfonts}       % blackboard math symbols
\usepackage{nicefrac}       % compact symbols for 1/2, etc.
\usepackage{microtype}      % microtypography
\usepackage[capitalise]{cleveref}
\usepackage{graphicx}
\usepackage{subcaption}

\title{Adaptive Regret for Bandits Made Possible: \\ Two Queries Suffice}

\author{Zhou Lu\thanks{Equal contribution}\\
Princeton University\\
\And
Qiuyi Zhang\footnotemark[1]\\
Google Deepmind\\
\And
Xinyi Chen\\
Google Deepmind\\
Princeton University\\
\AND
Fred Zhang\\
UC Berkeley\\
\And
David Woodruff\\
Google Research\\
Carnegie Mellon University\\
\And
Elad Hazan\\
Google Deepmind\\
Princeton University
}

\iclrfinalcopy % Uncomment for camera-ready version, but NOT for submission.
\begin{document}

\maketitle

\begin{abstract}
Fast changing states or volatile environments pose a significant challenge to online optimization, which needs to perform rapid adaptation under limited observation. In this paper, we give query and regret optimal bandit algorithms under the strict notion of strongly adaptive regret, which measures the maximum regret over any contiguous interval  $I$. Due to its worst-case nature, there is an almost-linear $\Omega(|I|^{1-\epsilon})$ regret lower bound, when only one query per round is allowed [Daniely el al, ICML 2015]. Surprisingly, with just two queries per round, we give Strongly Adaptive Bandit Learner (StABL) that achieves $\widetilde{O}(\sqrt{n|I|})$ adaptive regret for multi-armed bandits with $n$ arms. The bound is tight and  cannot be improved in general. Our algorithm leverages a  multiplicative update scheme of varying stepsizes and a carefully chosen observation distribution to control the variance. Furthermore, we extend our results and provide optimal algorithms in the bandit convex optimization setting. Finally, we empirically demonstrate the superior performance of our algorithms under volatile environments and for downstream tasks, such as algorithm selection for hyperparameter optimization.
\end{abstract}

% \input{ICLRebuttal}
% \newpage

\section{Introduction}
% \eh{the main thing remaining to do is experimentation, and check if a factor of  root(d) can be improved} 

In  online optimization, a player iteratively chooses a point from a decision set, and receives loss from an adversarially chosen loss function. The classic metric for measuring the performance of the player is regret: the difference between her total loss and that of the best fixed comparator in hindsight. 

However, as pointed out by Hazan and Seshadhri \cite{hazan2009efficient}, regret incentivizes static behavior and is not the correct metric in changing environments. They instead proposed the notion of \textit{adaptive regret}, defined as  the maximum regret over any continuous interval in time.
% 
% and produces $O(\sqrt{T})$ regret bounds on the regret with respect to the best comparator in the interval  \cite{hazan2009efficient}.  
This notion intuitively captures adaptivity  to the environment, and is studied later by \cite{daniely2015strongly} which gives an algorithm with near-optimal ${O}(\sqrt{|I|} \log T)$ adaptive regret for any interval $I$ simultaneously.

Most algorithms for  minimizing adaptive regret, such as \cite{hazan2009efficient, daniely2015strongly},  work by running  $O(\log T)$  independent copies of online learning schemes and using a meta-learner to aggregate them.
Therefore, they make $O(\log T)$ queries per round. In practice,  oftentimes it is expensive to evaluate the arms. For example, an arm corresponds to a set of hyperparameter  values of a large machine learning model, and playing an arm amounts to evaluating   the model performance and fully training the model.   This motivates the study of  query-efficient adaptive regret algorithms. 
Recently, Lu and Hazan \cite{lu2022efficient} give an improved algorithm with only $O(\log \log T)$  queries per round. 
However, the known lower bound in \cite{daniely2015strongly} only shows that in the bandit setting, near-optimal adaptive regret is impossible with one query. This leaves the following question:
\begin{center}
  \textit{Can we design algorithms with near-optimal adaptive regret using fewer queries?}
\end{center}
In this paper we give an affirmative answer, showing that two queries per round is enough to guarantee an $\widetilde{O}(\sqrt{|I|})$ adaptive regret bound. The key ingredient of our algorithm is using the additional query for exploration under a special distribution on arms, which serves to construct unbiased loss estimators with small variance for both experts and the meta-learner.

\subsection{Our Results}
We design adaptive regret minimization schemes with the power to query additional arms.  
In the adversarial multi-arm bandit (MAB) problem,  
the algorithm selects and plays one arm per round.
In our setting the algorithm can pick additional arms to query \textit{in parallel to} selecting which arm to play, then the loss values of both the queried and played arms will be revealed, but the algorithm only suffers the loss of the arm played. This is similar to, but still different from the multi-point feedback model \cite{agarwal2010optimal}, which suffers the average loss of the arms. The query complexity counts the total arm observations received, including the one played.

First, we prove that two queries suffice for $\widetilde{O}(\sqrt{nI})$ adaptive regret in the MAB setting (we slightly abuse the notation $I$ to denote its length $|I|$). 
Our algorithm runs an EXP3-type meta-algorithm on top of black-box bandit learners $\mA$ which are EXP3. It uses the additional query to perform exploration every round for constructing a bounded unbiased loss estimator. Then each bandit  algorithm together with the meta-algorithm gets updated according to the loss estimator. Together with the lower bound in \cite{daniely2015strongly}, our algorithm gives a tight characterization of the query efficiency of adaptive regret minimization. Such $O(\log T)$ to 2 improvement in the number of queries needed is significant for computational efficiency, for applications such as hyperparameter optimization in expensive settings.

\begin{table}[ht!]
\caption{Adaptive regret bounds and query efficiency in the adversarial multi-armed bandits setting. }
\begin{center}
\begin{tabular}{|c|c|c|}
\hline
 Algorithm &  
Adaptive regret bound &
Number of queries
\\
\hline
FLH \cite{hazan2009efficient}  & $\sqrt{nT }$ & $O(\log T)$ \\
\hline
SAOL \cite{daniely2015strongly}  & $ \sqrt{nI\log T} $ & $O(\log T)$  \\
%\hline
%\cite{daniely2015strongly}  & $\Omega(I^{1-\epsilon})$  & 1  \\
\hline
EFLH \cite{lu2022efficient} & $I^{\frac{1}{2}+\epsilon} \cdot  \sqrt{n\log T} $ & $O\left(\frac{\log \log T}{\epsilon}\right)$\\
\hline \textbf{This paper} (\cref{thm: bandit}) & $\sqrt{nI\log n} \cdot \log^{1.5} T $ & $2$\\
\hline
\end{tabular}

\label{table: bandit}
%\vskip -0.3in
\end{center}
\end{table}
Next, we extend our MAB method to the bandit convex optimization (BCO) setting, proving that three queries suffice for $\widetilde{O}(\sqrt{I})$ adaptive regret. We use one of the additional queries to do uniform exploration, which gets us the loss value of some expert $i$'s prediction $\ell_t\left(\mA_i(t)\right)$. This value is used to construct both a bounded unbiased loss estimator to the meta EXP3 algorithm, and an unbiased sparse gradient estimator with bounded variance by leveraging another query to get $\ell_t\left(\mA_i(t)+\mu \right)$, where $\mu$ is a small random perturbation term. 

Finally, we show empirically the advantage of using our adaptive algorithms in changing environments on synthetic and downstream tasks, such as algorithm selection during hyperparameter optimization.

\subsection{Related Work}
\paragraph{Adaptive Regret Minimization}

Motivated by early work on shifting experts \cite{herbster1998tracking,bousquet2002trackin}, the notion of adaptive regret was first proposed by  \cite{hazan2009efficient}. They gave an algorithm called Follow-the-Leading-History (FLH) with $O(\log^2 T )$ adaptive regret for   online   optimization on strongly convex functions.
The  bound on the adaptive
regret for general convex cost functions, however,  is $O(\sqrt{T}\log T)$. 
This question has been further studied by a sequence of work \cite{daniely2015strongly,jun2017improved,lu2022adaptive} which aims to improve the regret bound. In particular, \cite{daniely2015strongly} proposed the algorithm Strongly-Adaptive-Online-Learner (SAOL) achieving a near-optimal  $O(\sqrt{I}\log T)$ bound. Later, \cite{jun2017improved} improved this bound by a $\sqrt{\log T}$ factor and \cite{lu2022adaptive} attained a second-order bound. Crucially, note that all the algorithms use $O(\log T)$ queries per round.

Recently, \cite{lu2022efficient} initiated a study on improving the efficiency of adaptive regret minimization. They showed a trade-off between adaptive regret and computational complexity, that an $\tilde{O}(I^{\frac{1}{2}+\epsilon})$ regret bound is achievable with $O\left(\frac{\log \log T}{\epsilon}\right)$ base algorithms,each with one query per round. This result differs from ours in two aspects: (i) \cite{lu2022efficient} studied the computational complexity (number of base algorithms), while we consider the query complexity and achieve a better bound; and (ii) \cite{lu2022efficient} focused on improving the classic exponential-history-lookback technique of \cite{hazan2009efficient, daniely2015strongly} itself, by showing an regret upper bound with doubly exponential lookback. However, they also showed a matching lower bound that the history-lookback technique must use $\Omega(\log \log T)$ experts to achieve non-vacuous regret. Instead, our work bypasses such limitations by incorporating the EXP3 method into the classic framework, using only a constant number of queries.

For adaptive regret in the BCO setting, \cite{zhao2021bandit} showed that an $\tilde{O}(\sqrt{T})$ adaptive regret bound is achievable under the two-point feedback model. However, the result is weaker than ours in two aspects. First, our result provides a  strongly adaptive regret bound $\tilde{O}(\sqrt{I})$ which depends on the interval-length, and is a stronger notion than adaptive regret. Second, our regret bound is independent of the condition number $\kappa$, while their regret bound is $\tilde{O}(dGD\kappa\sqrt{T})$. We will
discuss the comparison in more details later.

\paragraph{Dynamic Regret}
Dynamic regret is another related notation to handle changing environments in OCO, which aims to capture comparators with bounded total movement instead. Starting from \cite{zinkevich2003online} which provided an $\tilde{O}(\mathcal{P}\sqrt{T})$ dynamic regret bound of vanilla OGD where $\mathcal{P}$ denotes the path length, there have been many works on improving and applying dynamic regret bounds. \cite{zhao2020dynamic, zhang2017improved} achieved improved dependence on $\mathcal{P}$ under further assumptions. \cite{baby2021optimal, baby2022optimal} focused on the setting of exp-concave and strongyly-convex loss, showing that the classic adaptive regret algorithm FLH \cite{hazan2009efficient} guarantees $\tilde{O}(T^{\frac{1}{3}}\mathcal{P}^{\frac{2}{3}})$ dynamic regret. Some works study the relationship between adaptive regret and dynamic regret, for example \cite{zhang2018dynamic, zhang2020minimizing}.

\iffalse
\paragraph{Multi-Point  Bandit Feedback}
Our query model is similar to but still distinct from the multi-point bandit feedback model \cite{agarwal2010optimal}. In this model, the algorithm can play multiple arms each round, observe all their losses and  suffer the average. 
In this work, we only allow the algorithm to play a single arm per round and incur its loss. 
The loss value of additional  arms may be queried in parallel. 

\paragraph{Online Learning with Hints}
A recent   work studies online learning with probes \cite{BhaskaraGIKM23}, where the learner is allowed to observe the loss value of multiple arms before taking an action. This model differs from ours, since ours  does not permit  probing in advance. 
Another line of work studies online learning with hints, under various models different from our work \cite{dekel2017online,bhaskara2020online2,bhaskara2020online,bhaskara2021logarithmic}. Due to the power of hints, such a method usually permits a fast rate with log dependence on $T$ in the regret bound.
\fi

\paragraph{Bandit Convex Optimization}
The study on BCO was initiated by \cite{flaxman2004online}, which showed how to construct an approximate gradient estimator with bandit feedback, providing an $O(T^{3/4})$ regret bound. If multi-point feedback is allowed, the regret bound can be improved to $O(\sqrt{T})$ as shown in \cite{agarwal2010optimal}. Later, \cite{bubeck2017kernel} used a kernel-based algorithm to achieve the optimal $O(\sqrt{T})$ regret bound. Nevertheless, none of the prior works considered adaptive regret in the BCO setting.
\section{Settings and Preliminaries}
We study adaptive regret in a limited information model, for both the multi-armed bandit (MAB) and the continuous bandit convex optimization (BCO) problems. In the MAB problem, a decision maker $\mathcal{A}$ plays a game of $T$ rounds, she pulls an arm $x_t$ at round $t$ and chooses an additional set of arms $X_t$ to query. Then an adversary reveals the values of the loss vector $\ell_t$ on $\{x_t\}\cup X_t$.  The goal is to minimize the (strongly) adaptive regret, which examines   all contiguous intervals $I=[j,s]\in [T]$: (we will omit $I$ when it's clear from the context)
$$
\text{SA-regret($\mathcal{A}, I$)}=\max_{s-j=I}\left[ \sum_{t=j}^s \ell_t^{\top}e_{x_t}-\min_{i} \sum_{t=j}^s \ell_t^{\top}e_{i}\right].
$$
For the BCO setting, similarly the decision maker chooses $x_t \in \K\subset \reals^d$ at each time $t$ along with a set of points $X_t$, where $\K$ is some convex domain, then the adversary reveals the loss values on $\{x_t\}\cup X_t$. The player aims to minimize the (strongly) adaptive regret as well: (notice for expected regret, it contains the above definiton for MAB as a special case when $\K$ is a simplex)
$$
\text{SA-regret($\mathcal{A}, I$)}=\max_{s-j=I}\left[ \sum_{t=j}^s \ell_t(x_t)-\min_{x\in \K} \sum_{t=j}^s \ell_t(x)\right].
$$
% In the classic MAB setting with $X_t=\emptyset$, it has been shown that strongly adaptive regret suffers a  almost linear lower bound of $\Omega(I^{1-\epsilon})$ \cite{daniely2015strongly}.
The decision maker  has limited power to query $\ell_t$.   For every round $t$ she has a   query budget $|X_t|\le N$ for some small constant $N\in \mathbb{N}^{+}$. In particular, we will only consider the case $N=1$ for the MAB problem and $N=2$ for the BCO problem. We emphasize the choice of which arms to query is made before any information of the loss is revealed, thus differentiates our setting from the stronger "Bandit with Hints" setting \cite{BhaskaraGIKM23}. Our setting is stronger than the vanilla MAB setting, but weaker than the full-information setting or the hint setting.

We make the following assumption on the loss $\ell_t$ and domain $\K$, which is standard in literature. 
\begin{assumption}
    For the MAB problem, the loss vector $\ell_t$ is assumed to have each of its coordinate bounded in $[0,1]$. For the BCO problem, we assume the loss $\ell_t$ is convex, $G$-Lipschitz and non-negative. The convex domain $\K$ is sandwiched by two balls centered at origin: $r\mathbf{B}\subset\K\subset D\mathbf{B}$, we denote the condition number $\kappa=D/r$. 
\end{assumption}

\subsection{The Query Model}
Our query model is similar to the multi-point feedback model \cite{agarwal2010optimal}, with a slight difference. Both models specify a set of arms to query beforehand, then after the adversary reveals the loss, both models receive loss information on the set of chosen arms. Here, the definition (evaluation of the loss function at a given point) and complexity (number of evaluations) of query are the same for both models, and the only difference is that our model only incurs loss of the actually played arm, while the multi-point feedback model incurs the average loss of the set of chosen arms. Notice that neither model strictly implies the other model. As a result, it's fair to compare the query efficiency of our result with previous works under the multi-point feedback model. 

\subsection{The EXP3 Algorithm} 
The EXP3 algorithm for adversarial MAB \cite{auer2002nonstochastic} is based on the multiplicative update method, and performs a weight update according to an unbiased loss estimator using bandit feedback. Denote $w_t(k)$ as the weight of arm $k$ at time $t$, the EXP3 algorithm   samples the arm $x_t$ according to the weights $w_t(k)$. Ideally, we would like to update each weight based on full information of the loss:
$$
w_{t+1}(k)=w_t(k)\left(1-\eta \ell_t^{\top} e_k\right).
$$
In the bandit setting, we have only the value of $\ell_t^{\top} e_{x_t}$.
Let $\mathbf{1}$ denote the all-ones vector.
The EXP3 algorithm instead constructs an unbiased estimator $\hat{\ell}_t$ to replace $\ell_t$:
$$
\hat{\ell}_t(i)=\mathbf{1}_{i=x_t} \frac{\ell_t^{\top} e_{x_t}}{w_t(k)}.
$$
This pseudo loss has a bounded variance. In particular, $\mathbb{E}[w_t^{\top} \hat{\ell}_t^2]\le n$, which gives the EXP3 algorithm an $\tilde{O}(\sqrt{nT})$ regret bound.

\subsection{Standard Framework for Minimizing Adaptive Regret}
We briefly review the standard framework of adaptive regret minimization. Since adaptive regret   asks for small regret over all intervals, a simple idea is to construct, for any interval $I$, a base learner $\mA_I$ that achieves optimal regret $O(\sqrt{I})$ on $I$. There are $O(T^2)$  contiguous intervals between $[1,T]$. Thus, a naive procedure is to run a meta-learner, such as multiplicative weights (MW) update, on these $O(T^2)$ experts, and this leads to  an $O(\sqrt{I \log T})$ adaptive regret.

Unfortunately, although the regret bound is near optimal, the naive algorithm is   inefficient. In particular, for any $t$ there are $\Omega(T)$ number of intervals containing $t$, and the naive algorithm needs to maintain and update $\Omega(T)$ number of base learners per round.  To resolve this issue, the key technique is to consider only a subset $S$ of all intervals \cite{hazan2009efficient, daniely2015strongly}:
$$
S=\left\{[s2^k,(s+1)2^k-1] \mid 0\le k\le \log T,\ s\in \mathbb{N^{+}}, (s+1)2^k-1\le T\right\}.
$$
The set $S$, known as the geometric interval set, contains all successive intervals with length equal to some power of two. In particular, the $kth$ expert will only need to optimize the series of intervals $[2^k,2\times 2^k-1],[2\times 2^k,3\times 2^k-1],\cdots$ one by one. With this technique, the algorithm only needs to hold a running set of base algorithms with size $O(\log T)$, while maintaining a near-optimal $\tilde{O}(\sqrt{I})$ adaptive regret bound.
\section{Adaptive Regret in Multi-Armed Bandits}
In this section, we study efficient adaptive regret minimization in the adversarial multi-armed bandit (MAB) setting. We propose a procedure  (Algorithm \ref{alg bandit}) that achieves $\tilde{O}(\sqrt{nI})$ adaptive regret, using only two queries. 
At a high level, the algorithm maintains a set of instances of the EXP3 algorithm (i.e., base learners) and aggregates them via a MWU meta-algorithm. 
% Together with the lower bound of \cite{daniely2015strongly}, this gives a tight characterization of the efficiency-regret trade-off in the MAB setting.

\begin{algorithm}[ht!]
\caption{Strongly Adaptive Bandit Learner (StABL)}
\label{alg bandit}
\begin{algorithmic}[1]
\STATE \textbf{Input:} general EXP3 algorithm $\mathcal{A}$ and horizon $T$. 
\STATE Construct interval set $S=\{[s2^k,(s+1)2^k-1] \mid 2+\log\log T\le k\le \log T, s\in \mathbb{N^{+}}\}$.
\STATE Construct $B=\log T-(1+\log\log T)$ independent instances of EXP3 algorithm $\mA_k$, where $\mA_k$ optimizes each $\{I\in S| 2^k=|I|\}$ one after another since they don't overlap.
\STATE Denote $w_t(k)$ to be the weight assigned to $\mA_k$ at time $t$ by the meta-algorithm.
\STATE Denote $v(t,k)\in \mathbb{R}^n$ to be the distribution over arms by $\mA_k$ at time $t$, and $v(t, k)_i$ to be the probability of sampling arm $i$ by $\mA_k$ at time $t$.
\STATE Define $\eta_k=\min\left\{1/2\sqrt{n},1/\sqrt{n|2^k|}\right\}$, and initialize $w_1(k)=\eta_k$ for all $k\in [B]$.
\FOR{$\tau = 1, \ldots, T$}
\STATE Let $W_t=\sum_k w_t(k)$ and   $p(t)=\frac{1}{W_t}(...,w_t(k),...)$  be the distribution over the base learners.
\STATE For all $i\in [n]$, let $$P(t)_i=\frac{\max_{k} v(t,k)_i^2}{2\sum_{i} \max_{k} v(t,k)_i^2} +\frac{\sum_k v(t,k)_i}{2B}$$ \\
\texttt{// This defines a probability distribution over $n$ arms.}
\STATE Sample $x_t \sim \sum_k p(t)_k v(t,k)$, and in parallel sample $x'_t\sim P(t)$.
\\ \texttt{// Only the second sample $x'_t$ will be used for weight updating.}
\STATE Play $x_t $, suffer loss $\ell_t^{\top} e_{x_t} $ and observe loss $\ell_t^{\top} e_{x'_t}$. Compute loss estimator 
$$\hat{\ell}_t=\mathbf{1}_{i=x'_t} \frac{1}{P(t)_{x'_t}} \ell_t^{\top}e_{x'_t}.$$ 
\STATE Update the weight $v(t+1,k)$ of each EXP3 instance with loss estimator $\hat{\ell}_t$, via Algorithm \ref{alg:exp3}.
\STATE Update the meta-algorithm's weights over base learners via the loss estimator $\hat{\ell}_t$. For each $k$, update $w_{t+1}(k)$ as follows,
$$
w_{t+1}(k)=\left\{
\begin{array}{lcl}
{\eta_k} & & {2^k | t+1}\\
{w_{t}(k)\left(1+\eta_k \tilde{r}_{t}(k)\right)} & & \textbf{else}
\end{array}\right.
$$
where $\tilde{r}_{t}(k)=\hat{\ell}_t^{\top} \sum_k p(t)_k v(t,k)  -\hat{\ell}_t^{\top}v(t,k)$.
\ENDFOR
\end{algorithmic}
\end{algorithm}

\begin{algorithm}[ht!]
\caption{Sub-routine: EXP3 with a General Loss Estimator}
\label{alg:exp3}
\begin{algorithmic}[1]
\STATE Input: horizon $T$, learning rate $\eta$ and $w_1=\frac{\mathbf{1}}{n}$.
\FOR{$t = 1, \ldots, T$}
\STATE Play $i_t \sim w_t$.
\STATE Get unbiased estimator of loss $\tilde{\ell}_t$.
\STATE Update $y_{t+1}(i)=w_t(i)e^{-\eta \tilde{\ell}_t(i)}$, $w_{t+1}=\frac{y_{t+1}}{\|y_{t+1}\|_1}$.
\ENDFOR
\end{algorithmic}
\end{algorithm}

\begin{theorem}[Adaptive regret minimization for multi-armed bandits]
\label{thm: bandit}
For the multi-armed bandits problem with $n$ arms and $T$ rounds,  Algorithm \ref{alg bandit} achieves an expected adaptive regret bound of $O\left(\sqrt{nI\log n } \log^{1.5} T\right)$, using two queries per round.
\end{theorem}

%\subsection{Algorithm and Main Idea}

The main idea is to use  EXP3-type algorithms for both the black-box base learners and the meta-algorithm  in the typical adaptive regret framework. Directly using EXP3 in the MAB setting will fail, because the weight distribution   might become unbalanced over time and the unbiased estimator of the loss will have huge variance. 
When the unbiased estimator is propagated to the meta algorithm, it serves now as the value of ``arm'' which leads to a sub-optimal regret. 

The remedy is to use an additional evaluation to create unbiased estimators of the loss vector with controllable variance, for updating both experts and the meta-algorithm. A na\"ive (but sub-optimal) choice is to sample uniformly over the arms to observe $x'_t$, and update both experts and the meta-algorithm (line 10 in Algorithm \ref{alg bandit}). 
By doing this, the loss estimator $\hat{\ell}_t$ is not only unbiased, but also norm bounded by $n$. As a result, the variance term in the classical EXP3 analysis is   bounded by $n^2$. Thus, the regret of the base learners,  as well as the meta-algorithm, are    $\sqrt{n}$ factor worse than optimal.

Instead of  the na\"ive uniform exploration, we choose the  distribution of the additional query  to be the average of a uniform exploration among all  base learners' choices and a near-optimal distribution for the meta-algorithm. Below we briefly explain the design of this importance sampling $P(t)_i$ (line 9 in Algorithm \ref{alg bandit}). $P_t$ is the average of two distributions, the first one controlling the regret of the meta learner while the second one controlling the regret of base learners. If $P_t$ is merely the second distribution, it achieves near-optimal regret for base learners, with the excessive term $\sqrt{n}$ improved to $\sqrt{\log T}$. However, the second distribution alone will make the regret of the meta learner unbounded, thus we mix it with the first distribution which is designed to control the regret of the meta learner. Such mixing is known to only affect the regret by a constant for EXP-3 type algorithms.

\subsection{Proof Sketch}
In the following, we slightly abuse notation and let $\ell_t^\top x = \ell_t^\top e_x$ for $x\in[n]$.
The proof consists of three steps: decomposing the randomness of playing $x_t$, bounding the base learners' regret, and analyzing the regret of the meta-algorithm.

\textbf{Step 1:} There are two sources of randomness in the algorithm, namely, randomness $\textbf{pl}$ in sampling which arm $x_t$ to play, and $\textbf{ob}$ in sampling which arm $x'_t$ to observe and update weights. The key observation is that $\textbf{pl}$ is independent of $\textbf{ob}$. Thus, we have the following equivalence via linearity and the tower property of expectations, for any fixed interval $I$ and arm $x^*$,
$$
\mathbb{E}_{\textbf{pl,ob}}\left[ \sum_{t\in I} \ell_t^{\top} x_t-\sum_{t\in I} \ell_t^{\top} x^*\right]=\mathbb{E}_{\textbf{ob}}\left[ \sum_{t\in I} \hat{\ell}_t^{\top}  \sum_k p(t)_k v(t,k)-\sum_{t\in I} \hat{\ell}_t^{\top} x^*\right],
$$
which decouples the randomness of playing $x_t$ from observing $x'_t$.

\textbf{Step 2:} We prove the following lemma on EXP3 algorithms with general unbiased loss estimators, which gives a regret guarantee for Algorithm \ref{alg:exp3}.

\begin{lemma}[Regret for EXP3]
Given $\tilde{\ell}_t$, an unbiased estimator of $\ell_t$, such that for some distribution $z_t$, $\tilde{\ell}_t(i) = \frac{1}{z_t(i)} \ell_t(i)$ and $\tilde{\ell}_t(j) = 0$ for $j\neq i$ with probability $z_t(i)$. Suppose $z_t$ satisfies $w_t(i) \le Cz_t(i)$ for all $i$, where $w_t(i)$ represents the weight of the $ith$ expert at time $t$. Algorithm \ref{alg:exp3} using $\tilde{\ell}_t$ with $\eta=\sqrt{\frac{\log n}{T n C}}$ has regret bound $2\sqrt{ CnT \log n}$.
\end{lemma}
As a result, noticing $C=2\log T$ in our algorithm, we have that for any expert $\mA_k$, its regret on interval $I$ can be bounded by $O(\sqrt{nI\log n \log T})$.

\textbf{Step 3:} We analyze the regret of the meta-learner, which is $\mathbb{E}_{\textbf{pl, ob}}\left[ \sum_{t\in I} \ell_t^{\top} x_t-\sum_{t\in I} \ell_t^{\top} v(t,k)\right]=\mathbb{E}_{\textbf{ob}}\left[ \sum_{t\in I} \tilde{r}_t(k)\right]$. Define the pseudo-weight $\tilde{w}_t(k)$ to be $\tilde{w}_t(k)=\frac{w_t(k)}{\eta_k}$, and $\tilde{W}_t=\sum_k \tilde{w}_t(k)$, we first prove $\tilde{W}_t\le t(\log t +1)$ using induction, which leads to the following estimate
$$
\mathbb{E}_{\textbf{ob}}\left[\sum_{t\in I}\tilde{r}_t(k)\right] \le \eta_k\mathbb{E}_{\textbf{ob}}\left[ \sum_{t\in I}\tilde{r}_t^2(k)\right]+\frac{2 \log T}{\eta_k}.
$$
Then we show the term $\mathbb{E}_{\textbf{ob}}\left[ \tilde{r}_t^2(k)\right]$ can be bounded by:
$$
\mathbb{E}_{\textbf{ob}}\left[ \tilde{r}_t^2(k)\right]\le \mathbb{E}_{\textbf{ob},<t}\left[\sum_{i=1}^n  \frac{(\ell_t^{\top}e_i)^2}{P(t)_i} \left(\max_k e_i^{\top}v(t,k)\right)^2 \right]\le 2n \log T,
$$
which implies the regret of our algorithm on any interval $I\in S$, can be bounded by $O(\log T \sqrt{nI \log n})$. Finally, such regret can be extended to any interval at the cost of an additional $\sqrt{\log T}$ term, by Cauchy-Schwarz (each arbitrary interval can be decomposed into a disjoint union of 
$O(\log T)$ geometric intervals as in \cite{daniely2015strongly}).

\section{Adaptive Regret in the BCO Setting}\label{sec:bco}
The result from the previous section inspires us to consider the following question: can we use a similar approach to achieve near-optimal adaptive regret in the bandit convex optimization (BCO) setting with constant number of queries? The answer is yes, that near-optimal adaptive regret is achievable with three queries, at the cost of extra logarithmic terms in the regret.

The algorithm for adaptive regret in the BCO setting (see appendix) has a similar spirit as Algorithm \ref{alg bandit}. At time $t$, each expert $k$ asks for gradient estimation at a point $\mA_k(t)$, and we only randomly sample one expert to get the gradient estimation. If we sample expert $k_t$, the estimation for $\mA_{k_t}$ is 
$$
\frac{dm\mathbf{u}(\ell_t(\mA_{k_t}(t)+\delta_t \mathbf{u})-\ell_t(\mA_{k_t}(t)))}{\delta_t},
$$
and the gradient estimation for the rest of experts is 0, where $m$ is the number of experts, $\delta$ is a small constant and $\mathbf{u}$ is a random unit vector. Such estimation is classic in the two-point feedback BCO setting \cite{agarwal2010optimal}, in fact all we need is a (nearly) unbiased estimation of loss with bounded absolute value. Then we can use the same approach as the MAB case to make an unbiased estimation of the experts' losses, by setting
\begin{align*}
\tilde{\ell}_t(\mA_k(t))=m \ell_t(\mA_{k_t}(t)) 1_{[k=k_t]}.
\end{align*}
We have the following regret guarantee for Algorithm 3. The proof is delayed to the appendix.

\begin{theorem}\label{thm: bco}
In the BCO setting, Algorithm 3 with $\delta_t=\frac{1}{\kappa T}$ uses three queries per round and achieves an expected adaptive regret bound of 
\[O\left(dGD\sqrt{I}\log^2 T \right). \]
\end{theorem}

Theorem \ref{thm: bco} also directly applies to the full-information setting, improving the query complexity of all previous works in adaptive regret from $\Theta(\log T)$ to a constant, while preserving the optimal regret.

\iffalse
Compared to \cite{zhao2021bandit}, note that though their algorithm uses one fewer query, their regret bound is weaker. First, they considered only the weaker notion of adaptive regret, failing to achieve dependence on interval length which leads to a worse $\sqrt{T}$ dependence on horizon. Second, their regret bound $\tilde{O}(dGD\kappa\sqrt{T})$ has a worse dependence on the condition number $\kappa$. As a comparison, our algorithm has no such dependence by limiting the effect of $\kappa$ to only the cost of using $\hat{\K}$ and eliminating the dependence by the choice of $\delta$. We discuss how to further reduce the required number of queries by one in the appendix, by using the linear surrogate loss idea of \cite{zhao2021bandit} in our algorithm. 

\xc{From their bound it seems like they can also get ride of the $\kappa$ dependence by setting $\delta=1/\kappa T$?}
\fi

Compared to \cite{zhao2021bandit}, note that though their algorithm uses one fewer query, they considered only the weaker notion of adaptive regret, failing to achieve dependence on interval length which leads to a worse $\sqrt{T}$ dependence on horizon. We discuss how to further reduce the required number of queries by one for strongly-adaptive regret in the appendix, by using the linear surrogate loss idea of \cite{zhao2021bandit} in our algorithm. 

\section{Experiments}
In this section, we evaluate the proposed algorithms on synthetic data and the downstream task of hyperparameter optimization. We note that even though the algorithms are stated in terms of losses, we find that in practice, using rewards instead of losses lead to better performance.  

\subsection{Learning from Expert Advice}
\paragraph{Experimental Setup} We first consider the learning with expert advice setting with linear rewards, and demonstrate the advantage of using an adaptive algorithm in changing environments. For simplicity, the arms and experts are equivalent in our construction. More specifically, there are $N$ arms, indexed from $0$ to $N-1$, and there are $N$ experts, where the $i$-th expert (zero-indexed) always suggests pulling the arm $i$. We take $N=30$, time horizon $T=4096$, and for each time step, we randomly generate a baseline reward $\tilde{r}_t \in \reals^N$ where $\tilde{r}_{t, i}$ is drawn from the uniform distribution over $[0, 0.5)$. To simulate the volatile nature of the environment, we divide the time horizon into $4$ intervals, and add an additional reward to the baseline reward of a different expert in each interval. In particular, for  $t\in [0, 1023]$, we set the reward $r_{t, 0} = \tilde{r}_{t, 0} + 0.5$ and $r_{t, i} = \tilde{r}_{t, i}$ for $i \neq 0$; for $t\in [1024, 2047]$, we do the same for expert $1$, and so on. The reward at each time step is $r_t^\top x_t$, where $x_t$ is the point played. 

\paragraph{Analysis} In the first experiment, we compare the performance of the non-adaptive EXP3 algorithm with uniform exploration, StABL, and the StABL algorithm with naive exploration. In StABL with naive exploration, the observation query is sampled with the uniform distribution (Line 10 in Algorithm \ref{alg bandit} is replaced with the uniform distribution). We run all algorithms 5 times and plot the moving average of their reward with a window of size 50 in the left subfigure of Figure \ref{random_rewards}. Both StABL and and StABL Naive can adapt to changing environments and recover after the best arm has changed. In contrast, EXP3 is optimal in the first interval, but its performance quickly degrades in the subsequent intervals. Compared to StABL, StABL Naive is worse in most intervals, reflecting the suboptimal performance resulting from the naive observation sampler. 

% This variant of the bandit algorithm has a worse regret guarantee compared to Algorithm \ref{alg bandit} by a factor of $\sqrt{n}$, and we compare the two variants to investigate the importance of designing the appropriate observation sampler. 

\begin{figure}
\centering
\begin{subfigure}{0.5\textwidth}
  \centering
  \includegraphics[width=.8\linewidth]
  {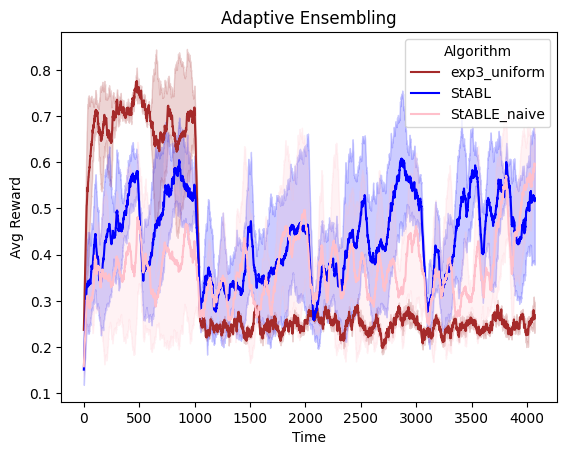}
  \label{fig:sub1}
\end{subfigure}%
\begin{subfigure}{0.5\textwidth}
  \centering
  \includegraphics[width=.8\linewidth]{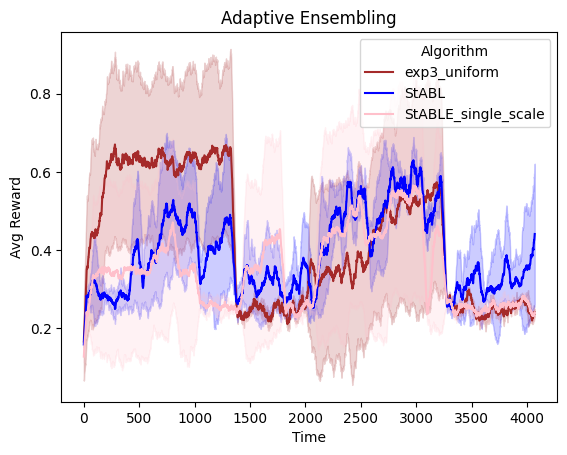}
  \label{fig:sub2}
\end{subfigure}
\caption{Comparison plots of the algorithm rewards in the learning with expert advice setting. The right subfigure shows the performance of the algorithms when the best arm changes at random intervals, and demonstrates the advantage of using base algorithms with varying history lengths }
% The left subfigure corresponds to the first experiment, where we investigate the advantage of sampling the observation using our custom distribution, instead of the uniform distribution. The right subfigure shows the performance of the algorithms when the best arm changes at random intervals, and demonstrates the advantage of using base algorithms with varying history lengths.}
\label{random_rewards}
\end{figure}

To understand the importance of having base algorithms with varying history lengths, we also compare StABL with a variant, StABL Single Scale, where the base algorithm only has one history length of 1024. Such a variant would perform well in the previous experiment, since the intervals have the same length. In this experiment, we randomly generated time steps at which the best arm changes, and they are 1355, 1437, 1798, 3249, for a time horizon of 4096. As before, we run each algorithm 5 times and compare the moving average reward in the right subfigure of Figure \ref{random_rewards}. In this experiment, StABL Single Scale struggles in intervals that do not conform to a length of 1024. Thus, even though StABL requires more memory and compute by running $\log T$ base algorithms, it is more robust to volatile environments that can have irregular changes. 

% The fourth interval has a length of 1451, and StABL Single Scale has worse performance in this interval. 

\subsection{Algorithm Selection for Hyperparameter Optimization}

\paragraph{Algorithm as Arms} We compare the effectiveness of bandit-feedback regret minimization algorithms for the downstream task of choosing the best algorithms for minimizing a blackbox function. Specifically, we view each round as a selection process between an ensemble of evolutionary strategies, each with different algorithm parameters, such as perturbation, gravity, visibility and pool size \cite{yang2013firefly}. Therefore each arm represents a specific evolutionary algorithm and the rewards of each arm are observed per-round in a bandit fashion, indicating the performance of the algorithm at that round, given all of the data observed so far. 

\paragraph{Benchmark and Reward Signal} The reward is determined by the underlying task of blackbox optimization, which is done on Black-Box Optimization Benchmark (BBOB) functions \citep{tuvsar2016coco}. BBOB functions are usually non-negative and we are in the simple regret setting, in which we want to find some $x$ such that $f(x)$ is minimized. Therefore we use the decrease of the objective, upon evaluation of the algorithm's suggestion, as the reward. Specifically, if $\mathcal{D}_i = (x_i, f(x_i))$ are all the evaluations so far in round $k$, then each algorithm $\mathcal{A}_j$ suggests $x_j = \mathcal{A}_j (\mathcal{D}_j)$ and the corresponding reward is $r_j = \max(\min_i (f(x_i)) - f(x_j), 0)$. Note that this reward is always positive and we normalize our BBOB functions so that it is within $[0, 1]$. Note that this reward signal is sparse, so we also add a regularization term of $-\lambda f(x_j)$ with $\lambda = 0.01$ of the negative objective into the reward. In addition to normalization, we also apply random vectorized shifts and rotations on these functions, as well as adding observation noise to the evaluations.

\paragraph{Analysis}
We ran the four algorithms: the uniform random algorithm, classic EXP3 algorithm, the EXP3 algorithm with uniform exploration, and StABL with history lengths $[20, 40, 80, 160, 320, 640, 1280, 2560]$. We run each of our algorithms in dimensions $d = 32, 64$ and optimize for $1000, 2000$ iterations with $5$ repeats. For metrics, we use the log objective curve, as well as the performance profile score, a gold standard in optimization benchmarking (higher is better) \cite{dolan2002benchmarking}. Our full results are in the appendix and here, we focus on the SPHERE function (\cref{fig:sphere}), where we see that StABL generally performs better adaptation than the EXP3 algorithms leading to better optimization throughout. Specifically, for the first half of optimization, both StABL and EXP3-Uniform enjoy a consistent performance advantage in performance, with respect to the Uniform strategy. However, this advantage sharply drops after around 600 Trials, after which the EXP3 strategy is no longer competitive against Uniform and ends up worse, as its performance curve abruptly stops at Trial 700. StABL generally tracks with the EXP3 strategy throughout this process but is noticeably better at adapting and maintaining its advantage after the critical turning point.

\begin{figure}[htb!]
\centering
\includegraphics[width=.8\linewidth]{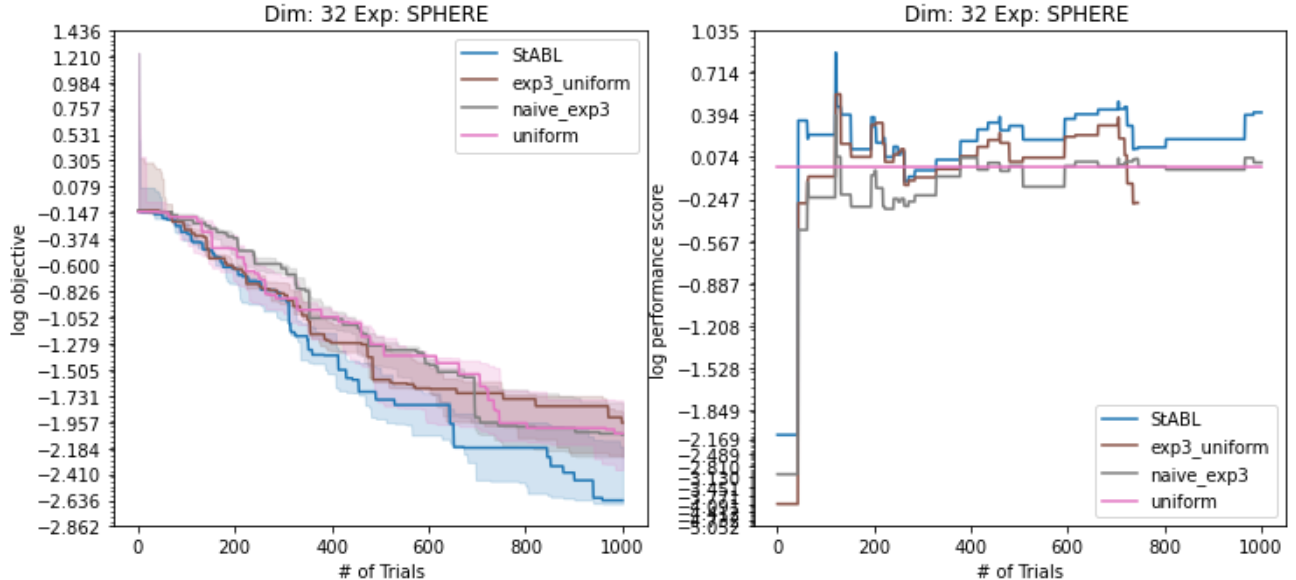}
\caption{Algorithm comparison plots of the log objective (lower is better) and the performance profile score against the Uniform baseline (higher is better) for minimizing the 32-dimensional SPHERE across 1000 trials.}
\label{fig:sphere} 
\end{figure}

% While both EXP3-Uniform and StABL is better than Uniform for the first $600$ trials, there is a sharp decline after which EXP3-Uniform never recovers and ends up worse than Uniform (performance curve stops). StABL suffers a small performance hit but quickly adapts to maintain a clear advantage throughout. 
\section{Conclusions}
We study adaptive regret in the limited observation model. By making merely one additional query with a carefully chosen distribution, our algorithm achieves the optimal $\tilde{O}(\sqrt{nI})$ adaptive regret bound in the multi-arm bandit setting. This result not only improves the state-of-the-art query efficiency of $O(\log \log T)$ in \cite{lu2022efficient}, but matches the lower bound in the bandit setting \cite{daniely2015strongly}, thus providing a sharp characterization of the query efficiency of adaptive regret. As an extension, we prove that the optimal $\tilde{O}(\sqrt{I})$ adaptive regret can be achieved in the bandit convex optimization setting, with only two additional queries. We also conduct experiments to demonstrate the power of our algorithms under changing environments and for downstream tasks. We list some limitations of this work and potential directions for future research, elaborated below.

\textbf{Reduce the logarithmic dependence in $T$:} Although our regret bound is tight in the leading parameters $n,I$, it involves a worse dependence on $\log T$. For the full-information setting, the best known result in \cite{orabona2016coin} only contains an $\sqrt{\log T}$ dependence. The main open question is thus: can the $\log^{1.5} T$ factor be improved in general?

\textbf{Fewer queries for BCO:} Our BCO Algorithm uses two additional queries to achieve the optimal adaptive regret bound, however for the MAB setting Algorithm \ref{alg bandit} requires only one query. Though BCO is a harder problem than MAB, and the lower bound of \cite{daniely2015strongly} was designed for MAB, we wonder whether an algorithm can use even fewer queries in the BCO setting.

\textbf{Extension to dynamic regret:} In the full-information  setting, any advance in adaptive regret of exp-concave loss implies improved dynamic regret bounds via a black-box reduction, as shown in \cite{lu2022efficient}. However, it seems harder in the bandit setting. The fundamental difficulty is, for the meta-learner, how to exploit exp-concavity (or strong convexity) of the loss in the MAB setting.

\bibliography{references}
\bibliographystyle{iclr2024_conference}

% \iffalse
\newpage
\appendix
\section{BCO Algorithms}

\begin{algorithm}[ht!]
\caption{Strongly Adaptive BCO Learner}
\label{alg bco}
\begin{algorithmic}[1]
\STATE \textbf{Input:} OCO algorithm $\mA$ (Algorithm \ref{alg fkm}) and horizon $T$. 
\STATE Construct interval set $S=\{[s2^k,(s+1)2^k-1] \mid 2+\log\log T\le k\le \log T, s\in \mathbb{N}^+\}$.
\STATE Construct $B=\log T-(1+2\log\log T)$ independent instances of the expert BCO algorithm $\mA_k$, where $\mA_k$ optimizes each $\{I\in S| 2^k=|I|\}$ one after another since they don't overlap.
\STATE Denote $w_t(k)$ to be the weight assigned to $\mA_k$ at time $t$ by the meta algorithm.
\STATE Define $\eta_k=\frac{1}{GD}\min \left\{\frac{1}{2},\sqrt{\frac{\log T}{2^k}}\right\}$, and initialize $w_1(k)=\eta_k$.
\FOR{$\tau = 1, \ldots, T$}
\STATE Let $W_t=\sum_k w_t(k)$, denote $p(t)=\frac{1}{W_t}(...,w_t(k),...)$ as the distribution over experts.
\STATE Denote $\mA_k(t)$ to be the prediction of expert $\mA_k$ at time $t$.
\STATE Sample $k_t$ uniformly from $[B]$, sample a random unit vector $\mathbf{u}$ and choose constant $\delta_t$.
\STATE Play $x_t = \sum_k p(t)_k \mA_k(t)$ and suffer loss $\ell_t(x_t) $. Observe losses $\ell_t(\mA_{k_t}(t)+\delta_t \mathbf{u})$ and $\ell_t(\mA_{k_t}(t))$.
\STATE Updating experts: construct gradient estimation for each $\mA_{k}$ as below, then invoke Algorithm \ref{alg fkm}
$$
\frac{d\log T \mathbf{u}(\ell_t(\mA_{k_t}(t)+\delta_t \mathbf{u})-\ell_t(\mA_{k_t}(t)))}{\delta_t} \mathbf{1}_{[k=k_t]}
$$
\STATE Updating weights: construct loss vector estimations
$$
\tilde{\ell}_t(\mA_{k}(t))=\log T \ell_t(\mA_{k_t}(t)) \mathbf{1}_{[k=k_t]}, \ \tilde{\ell}_t(x_t)=\sum_{k} p(t)_k \tilde{\ell}_t(\mA_{k}(t))
$$
\STATE Update the meta algorithm's weights. For each $k$, update $w_{t+1}(k)$ as follows,
$$
w_{t+1}(k)=\left\{
\begin{array}{lcl}
{\eta_k} & & {2^k | t+1}\\
{w_{t}(k)(1+\eta_k \tilde{r}_{t}(k))} & & \textbf{else}
\end{array}\right.
$$
where $\tilde{r}_{t}(k)=\tilde{\ell}_t(x_t)-\tilde{\ell}_t\left(\mA_{k}(t)\right)$.
\ENDFOR
\end{algorithmic}
\end{algorithm}

\begin{algorithm}[ht!]
\caption{Sub-routine: BCO with General Gradient Estimator}
\label{alg fkm}
\begin{algorithmic}[1]
\STATE \textbf{Input}: horizon $T$, learning rate $\eta$ and $x_1\in \K$.
\FOR{$t = 1, \ldots, T$}
\STATE Play $x_t$ and suffer loss $\ell_t(x_t)$.
\STATE Get gradient estimator $\tilde{g}_t$, such that there exists $a,b>0$ and another proxy loss $\hat{\ell}_t$, satisfying:\newline
$\E[\tilde{g}_t]=\nabla \hat{\ell}_t(x_t)$, $\E[\|\tilde{g}_t\|_2]\le b$ and $\forall x\in \hat{\K}$ it holds $|\ell_t(x)-\hat{\ell}_t(x)|\le a$.
\STATE Update $x_{t+1}=\Pi_{\hat{\K}}[x_t-\eta \tilde{g}_t$], where $\hat{\K}=\{x|\frac{1}{1-\kappa\delta_t}x\in \K\}$ is the domain of $\hat{\ell_t}$.
\ENDFOR

\end{algorithmic}
\end{algorithm}

\newpage

\section{Proof of Theorem \ref{thm: bandit}}
\label{sec:pfthm1}
\begin{proof}
The proof consists of three steps: eliminating the randomness of playing $x_t$, analyzing the expert regret, and analyzing the meta algorithm regret.

\textbf{Step 1}

There are two sources of randomness in the algorithm, namely randomness in sampling which arm $x_t$ to play, and sampling which arm $x'_t$ to observe and update weights. Let $\textbf{pl, ob}$ denote the randomness of selecting arms to play and to observe, respectively, over all iterations $T$, and let $\E$ denote the unconditional expectation. 

The key observation here is that the randomness in playing is independent of the randomness in observing, affecting only the loss value suffered and nothing else. Based on this observation, the very first step of the proof is the following equivalence via linearity and tower property of expectation 

\begin{align*}
\mathbb{E}\left[ \sum_{t\in I} \ell_t^{\top} x_t-\sum_{t\in I} \ell_t^{\top} x^*\right]
&=\sum_{t\in I} \mathbb{E}\left[\E\left[ \ell_t^{\top} x_t- \ell_t^{\top} x^*|x_1, x_1', \ldots, x_{t-1}, x_{t-1}'\right]\right]\\
&=\sum_{t\in I} \mathbb{E}\left[\E_{\textbf{ob}}\E_{\textbf{pl}}\left[ \ell_t^{\top} x_t- \ell_t^{\top} x^*|x_1, x_1', \ldots, x_{t-1}, x_{t-1}'\right]\right]\\
&=\sum_{t\in I}\E\left[\mathbb{E}_{\textbf{ob}}\left[  \ell_t^{\top}  \sum_k p(t)_k v(t,k)- \ell_t^{\top} x^*|x_1, x_1', \ldots, x_{t-1}, x_{t-1}'\right]\right]\\
% &=\sum_{t\in I}\mathbb{E}_{\textbf{ob}, <t}\left[ \mathbb{E}_{\textbf{ob},t}\left[  \ell_t^{\top}  \sum_k p(t)_k v(t,k)- \ell_t^{\top} x^* \right]\arrowvert t-1 \right]\\
&=\sum_{t\in I}\mathbb{E}\left[ \mathbb{E}_{\textbf{ob}}\left[ \hat{\ell}_t^{\top}  \sum_k p(t)_k v(t,k)- \hat{\ell}_t^{\top} x^* |x_1, x_1', \ldots, x_{t-1}, x_{t-1}'\right]\right]\\
&=\sum_{t\in I}\mathbb{E}\left[ \mathbb{E}_{\textbf{ob}}\left[ \hat{\ell}_t^{\top}  \sum_k p(t)_k v(t,k)- \hat{\ell}_t^{\top} x^* |x_1', \ldots, x_{t-1}'\right]\right]\\
&=\sum_{t\in I}\mathbb{E}_{\textbf{ob}}\left[  \hat{\ell}_t^{\top}  \sum_k p(t)_k v(t,k)- \hat{\ell}_t^{\top} x^*\right]\\
&=\mathbb{E}_{\textbf{ob}}\left[ \sum_{t\in I} \hat{\ell}_t^{\top}  \sum_k p(t)_k v(t,k)-\sum_{t\in I} \hat{\ell}_t^{\top} x^*\right],
\end{align*}
which decouples the randomness of playing $x_t$ from observing $x'_t$. The third equality holds as we take expectation over $x_t$ conditioned on $x_1, x_1', \ldots, x_{t-1}, x_{t-1}'$, and the fourth equality holds since $\hat{\ell}_t$ is an unbiased estimator of $\ell_t$. In the fifth equality, we remove the conditioning over $x_1, \ldots x_{t-1}$, since the random variables inside the conditional expectation are not functions of these variables. Hence, we can take expectation only over the randomness of the observations. As a result of this equivalence, the random $x_t$ term in the regret is replaced by the convex combination of $v(t,k)$, which is exactly the expectation of $x_t$ over the randomness of playing.

\textbf{Step 2}

Assume that the loss value $\ell_t(i)$ of each arm $i$ at any time $t$ is bounded in $[0,1]$, to bound the regret of experts, we need the following lemma on EXP-3 algorithms with general unbiased loss estimators.

\begin{lemma}\label{lemma: exp}
Given $\tilde{\ell}_t$, an unbiased estimator of $\ell_t$, such that for some distribution $z_t$, $\tilde{\ell}_t(i) = \frac{1}{z_t(i)} \ell_t(i)$ and $\tilde{\ell}_t(j) = 0$ for $j\neq i$ with probability $z_t(i)$. Suppose $z_t$ satisfies $w_t(i) \le Cz_t(i)$ for all $i$. The EXP-3 algorithm using $\tilde{\ell}_t$ with $\eta=\sqrt{\frac{\log n}{T n C}}$ has regret bound $2\sqrt{ CnT \log n}$.
\end{lemma}

\begin{proof}
Following the standard analysis of EXP-3 in \cite{hazan2016introduction}, we have that
\begin{align*}
    \E[\text{regret}]= \E\left[\sum_{t=1}^T \ell_t(i_t)-\sum_{t=1}^T \ell_t(i^*)\right]&\le \E\left[\sum_{t=1}^T \tilde{\ell}_t(w_t)-\sum_{t=1}^T \tilde{\ell}_t(i^*)\right]\\
    &\le \E\left[\eta \sum_{t=1}^T \sum_{i=1}^n  \tilde{\ell}_t(i)^2 w_t(i)+\frac{\log n}{\eta}\right]\\
    &= \eta \sum_{t=1}^T \E\left[\sum_{i=1}^n  \tilde{\ell}_t(i)^2 w_t(i)\right]+\frac{\log n}{\eta}.
\end{align*}
Let $\E_t$ denote the expectation of a random variable conditioned on the randomness before time $t$ (before playing $i_t$). Notice that for each time step,
\begin{align*}
\E\left[\sum_{i=1}^n  \tilde{\ell}_t(i)^2 w_t(i)\right] =\E\left[\E_t\left[\sum_{i=1}^n  \tilde{\ell}_t(i)^2 w_t(i)\right]\arrowvert t-1\right]&=
\E\left[\sum_{i=1}^n \ell_t(i)^2 \frac{w_t(i) }{z_t(i)}\right]\\
&\le \E\left[\sum_{i=1}^n \frac{w_t(i) }{z_t(i)}\right]\le Cn
\end{align*}
Therefore, taking $\eta = \sqrt{\log n/TC n}$, we have
$$
\E[\text{regret}]\le 2\sqrt{TCn\log n}.$$

\end{proof}
 
Back to our algorithm, had we used $v(t,k)$ to sample $x'_t$, then in the above lemma we have $C=1$ for the $kth$ expert, which implies an optimal regret $2\sqrt{nI\log n}$. However, each expert may have a very different $v(t,k)$, thus making one of them optimal can lead to worse regret on the rest of experts.

Instead, in our algorithm we use the distribution $P_t\ge \frac{1}{2B} \sum_k v(t,k)$ to sample $x'_t$, in which $\frac{1}{B} \sum_k v(t,k)$ is just the average of $v(t,k)$. For each expert $k$, this distribution treated as $z_t$ in the above lemma, guarantees $C=2B\le 2\log T$. Therefore, we have that for any expert $k$, its regret on interval $I$ is bounded by
$$
\E[\text{regret of expert $k$}]= \mathbb{E}_{\textbf{ob}}\left[ \sum_{t\in I} \hat{\ell}_t^{\top} v(t,k)-\sum_{t\in I} \hat{\ell}_t^{\top} x^*\right]\le 2\sqrt{2nI\log n \log T}.
$$

\textbf{Step 3}

It's only left to check the regret of the meta algorithm. Our analysis follows Theorem 1 in \cite{daniely2015strongly} with a few changes. Recall the equivalence obtained in step 1, the term we need to bound is the regret of our action $x_t$ with respect to expert $k$ for all experts:

$$
\mathbb{E}_{\textbf{pl, ob}}\left[ \sum_{t\in I} \ell_t^{\top} x_t-\sum_{t\in I} \ell_t^{\top} v(t,k)\right]=\mathbb{E}_{\textbf{ob}}\left[ \sum_{t\in I} \tilde{r}_t(k)\right].
$$
where we recall the definition $\tilde{r}_{t}(k)=\hat{\ell}_t^{\top} \sum_k p(t)_k v(t,k)  -\hat{\ell}_t^{\top}v(t,k)$.

Define the pseudo-weight $\tilde{w}_t(k)$ to be $\tilde{w}_t(k)=\frac{w_t(k)}{\eta_k}$, and $\tilde{W}_t=\sum_k \tilde{w}_t(k)$. We are going to prove $\tilde{W}_t\le t(\log t +1)$. To this end, we fix any possible "trajectory" of $\tilde{W}_t$'s that can be reached by our algorithm, then prove $\tilde{W}_t$ is bounded using induction. Since the only source of randomness for the objective on the right hand side comes from the observations $x_1', \ldots, x_T'$, we take an arbitrary sequence of $x'_t$ such that the weights of the algorithm are totally deterministic.

When $t=1$, only $\mA_1$ is active therefore $\tilde{W}_1=1\le 1(\log 1+1)$. Assume now the claim holds for any $\tau\le t$, we decompose $\tilde{W}_{t+1}$ as
\begin{align*}
    \tilde{W}_{t+1}&=\sum_{k} \tilde{w}_{t+1}(k)\\
    &=\sum_{k, 2^k|t+1} \tilde{w}_{t+1}(k)+\sum_{k, 2^k\nmid t+1} \tilde{w}_{t+1}(k)\\
    &\le \log(t+1)+1+\sum_{k, 2^k\nmid t+1} \tilde{w}_{t+1}(k),
\end{align*}
because there are at most $\log(t+1)+1$ number of different intervals in $S$ starting at time $t+1$, where each such interval has initial weight $\tilde{w}_{t+1}(k)=1$. Now according to the induction hypothesis,
\begin{align*}
    \sum_{k, 2^k\nmid t+1} \tilde{w}_{t+1}(k)&=\sum_k \tilde{w}_t(k)(1+\eta_k \tilde{r}_t(k))\\
    &=\tilde{W}_t+\sum_k\tilde{w}_t(k)\eta_k \tilde{r}_t(k)\\
    &\le t(\log t +1)+\sum_k w_t(k) \tilde{r}_t(k).
\end{align*}

We complete the argument by showing that $\sum_k w_t(k) \tilde{r}_t(k)=0$. By the definition of $\tilde{r}_t(k)$, we have that
\begin{align*}
    \sum_k w_t(k) \tilde{r}_t(k)&=W_t\sum_k p(t)_k\left(\hat{\ell}_t^{\top}\left(\sum_k p(t)_k v(t,k)\right) -\hat{\ell}_t^{\top}v(t,k)\right)\\
    % &=W_t\left(\hat{\ell}_t^{\top}\left(\sum_k p(t)_k v(t,k)\right) -\left(\sum_k p(t)_k v(t,k)\right)\right)=0.
    &=W_t\left(\hat{\ell}_t^{\top}\left(\sum_k p(t)_k v(t,k) -\sum_k p(t)_k v(t,k)\right)\right)=0.
\end{align*}
The second inequality holds because $\sum_k p(t)_k = 1$.

We notice that the weights are all non-negative. To see this, notice the non-negativity of $\tilde{w}_t(k)$ is guarded by the non-negativity of $1+\eta_k \tilde{r}_t(k)$. Let's lower bound $\tilde{r}_t(k)$: by the definition of $\hat{\ell}_t$, it's a sparse vector where only the coordinate $x'_t$ can have a positive value $\frac{1}{P(t)_{x'_t}}$. However, $P(t)_{x'_t}$ is lower bounded by $\frac{v(t,k)_{x'_t}}{2B}$ by its definition (line 9), thus $\tilde{r}_t(k)\ge -2B\ge -2\log T$. Since $k\ge 2+\log\log T$, we have that $\eta_k\le \frac{1}{4\log T}$, which ensures the desired non-negativity.

Because weights are all non-negative, we obtain
$$
t(\log t+1) \ge \tilde{W}_{t}\ge \tilde{w}_t(k).
$$
Hence, using the inequality that $\log(1+x)\ge x-x^2$ for $x\ge -\frac{1}{2}$, we have that for the $k$ satisfying $2^k=|I|$,
\begin{align*}
    2\log t&\ge \log(\tilde{w}_t(k))=\sum_{t\in I}\log (1+\eta_k \tilde{r}_t(k))\\
    &\ge \sum_{t\in I}\eta_k \tilde{r}_t(k)-\sum_{t\in I}(\eta_k \tilde{r}_t(k))^2
\end{align*}
Rearranging the above inequality and taking an expectation over the observations, we get the following bound
$$
\mathbb{E}_{\textbf{ob}}\left[\sum_{t\in I}\tilde{r}_t(k)\right] \le \eta_k\mathbb{E}_{\textbf{ob}}\left[ \sum_{t\in I}\tilde{r}_t^2(k)\right]+\frac{2 \log T}{\eta_k}.
$$
Next, we need to estimate $\tilde{r}_t(k)$ with our observation, which we apply with the distribution $P(t)$. Denote $N_t=\sum_{i=1}^n \max_{k} v(t,k)_i^2$, the term $\mathbb{E}_{\textbf{ob}}\left[ \tilde{r}_t^2(k)\right]$ can be bounded by:
\begin{align*}
    \mathbb{E}_{\textbf{ob}}\left[ \tilde{r}_t^2(k)\right]&=\mathbb{E}_{\textbf{ob},<t}\left[\mathbb{E}_{\textbf{ob},t}\left[ \left(\hat{\ell}_t^{\top} \sum_k p(t)_k v(t,k)  -\hat{\ell}_t^{\top}v(t,k)\right)^2 \arrowvert t-1 \right] \right]\\
    &=\mathbb{E}_{\textbf{ob},<t}\left[\sum_{i=1}^n P(t)_i \frac{(\ell_t^{\top}e_i)^2}{P^2(t)_i} \left(\left(\sum_k p(t)_k v(t,k)\right) -v(t,k)\right)_i^2 \right]\\
    &\le \mathbb{E}_{\textbf{ob},<t}\left[\sum_{i=1}^n  \frac{(\ell_t^{\top}e_i)^2}{P(t)_i} \left(\max_k e_i^{\top}v(t,k)\right)^2 \right]\\
    &\le \mathbb{E}_{\textbf{ob},<t}\left[2n N_t \right].
\end{align*}
where we use the fact that $\frac{(\ell_t^\top e_i)^2}{P(t)_i} \leq \frac{2N_t}{\max_k e_i^{\top}v(t,k)^2} $ by the definition of $P(t)$. We still need to upper bound the normalization term $N_t$. In fact, since $v(t, k)$ is a distribution for each $k$, we can bound the Frobenius norm by: 
$$
N_t\le \sum_i \sum_k (e_i^{\top} v(t,k))^2\le \log T
$$
because $\sum_i e_i^{\top} v(t,k)=1$. As a result, we have that
$$
\mathbb{E}_{\textbf{ob}}\left[\sum_{t\in I}\tilde{r}_t(k)\right] \le 2\eta_k nI\log T+\frac{2 \log T}{\eta_k}\le 8\log T \sqrt{nI}.
$$
Putting the two pieces together, we have that the regret of our algorithm on any interval $I\in S$ with length at least $4\log T$ (i.e. $k\ge 2+\log\log T)$, can be bounded by $O(\log T \sqrt{nI \log n})$. For other intervals with length at most $4\log T$, the $O\left(\sqrt{nI\log n } \log^{1.5} T\right)$ regret bound automatically holds (this is the reason we only hedge over $2+\log\log T\le k\le \log T$ instead of $0\le k\le \log T$).

Finally, apply the same argument as in A.2 of  \cite{daniely2015strongly} (also Lemma 7 in \cite{lu2022adaptive}).  This allows us to extend the $O(\log T \sqrt{nI \log n})$ regret bound over $I\in S$ to any interval $I$ at the cost of an additional $\sqrt{\log T}$ term, by observing that any interval can be written as the union of at most $\log T$ number of disjoint intervals in $S$ and using Cauchy-Schwarz; see Theorem 1 in \cite{daniely2015strongly}.

\end{proof}

\section{Proof of Theorem \ref{thm: bco}}
\label{sec:pfthmbco}
\begin{proof}
Similar to the proof of Theorem \ref{thm: bandit}, we first prove a regret bound of Algorithm \ref{alg fkm}. We follow the classic framework of ``gradientless'' gradient descent (see chapter 6.4 in \cite{hazan2016introduction}). Let's focus on some expert $\mA_{k}$ over an interval $[j,s]$ where $2^k | j,s$, the proxy loss $\hat{\ell_t}$ for this expert is the smoothed version of $\ell_t$:
$$
\hat{\ell_t}(x)=\E_{\mathbf{u}\sim \mathbf{B}} \ell_t(x+\delta \mathbf{u})
$$
We can verify that the properties in Algorithm \ref{alg fkm} hold with constants $a=\delta G$ and $b=dG \log T$. As for the domain $\hat{\K}$, we notice that (1) $\forall x\in \hat{\K}, \mathbf{u}$ we have that $x+\delta \mathbf{u}\in \K$, therefore $\mA_{k_t}(t)+\delta_t \mathbf{u}$ in Algorithm \ref{alg bandit} is feasible; (2) for any point $x\in \K$, there exists another point $x'=\Pi_{\hat{\K}}[x]\in \hat{\K}$ such that $\|x-x'\|_2\le D\kappa\delta$ by the property of projection. These properties allow us to bound the regret of $\hat{\ell}_t$ instead:

\begin{align*}
    \sum_{t=j}^s \E[\ell_t(\mA_{k}(t))]- \sum_{t=j}^s \ell_t(x)&\le \sum_{t=j}^s \E[\ell_t(\mA_{k}(t))]-\sum_{t=j}^s \ell_t(x')+\kappa\delta DGT\\
    &\le \sum_{t=j}^s \E[\hat{\ell}_t(\mA_{k}(t))]-\sum_{t=j}^s \hat{\ell}_t(x')+(2+\kappa)\delta DGT.
\end{align*}

As a result, we only need to analyze the regret bound of Algorithm \ref{alg fkm} for the loss function $\hat{\ell}_t$. By the standard analysis of OGD (see Theorem 3.1 in \cite{hazan2016introduction}), we have that 
\begin{align*}
    \sum_{t=j}^s \E[\hat{\ell}_t(\mA_{k}(t))]-\sum_{t=j}^s \hat{\ell}_t(x')&\le \eta \sum_{t=j}^s \E\left[\|\tilde{g}_t\|_2^2\right]+\frac{D^2}{\eta}\\
    &\le \eta I d^2 G^2 \log^2 T+\frac{D^2}{\eta}
\end{align*}
Combining these two inequalities, we have that
$$
\sum_{t=j}^s \E[\ell_t(\mA_{k}(t))]- \sum_{t=j}^s \ell_t(x) \le \eta I d^2 G^2 \log^2 T+\frac{D^2}{\eta}+(2+\kappa)\delta DGT=O \left( dGD\sqrt{I}\log T\right)
$$
when we choose $\eta=\frac{D}{dG \sqrt{I} \log T}$ and $\delta=\frac{1}{\kappa T}$. This proves that each expert has a near-optimal expected regret bound over the intervals it focuses on. 

Notice that it's crucial to look at the function value $\ell_t(\mA_{k_t}(t))$ in Algorithm \ref{alg bco}, otherwise $b$ will have an additional $\frac{1}{\delta}$ dependence which leads to sub-optimal regret bounds. This corresponds to the sharp separation between two-query \cite{agarwal2010optimal} and one-query \cite{flaxman2004online} settings.

The rest of the proof is almost identical to that of Theorem \ref{thm: bandit}. Some details are even easier, since we only uniformly sample from the experts instead of using a complicated distribution as in Algorithm \ref{alg bandit}. The next step is to prove a regret bound of the meta MW algorithm. By a similar argument as step 1, for any interval $I=[j,s]\in S$, the actual regret over the optimal expert $\mA_{k}$ on $I$ can be transferred as
$$
 \E\left[\sum_{t=j}^s \ell_t(x_t)-\sum_{t=j}^s \ell_t(\mA_{k})\right]
    = \E\left[\sum_{t=j}^s \tilde{\ell}_t(x_t)-\sum_{t=j}^s \tilde{\ell}_t(\mA_{k})\right],
$$
which allows us to deal with the pesudo-loss $\tilde{\ell}_t$ instead. We focus on the case that $\sqrt{\frac{\log T}{2^k}} \le \frac{1}{2}$, because in the other case the length $I$ of the sub-interval is $O(\log T)$, and its regret is upper bounded by $IGD =O(GD\sqrt{\log T I})$, and the conclusion follows directly.

Define the pseudo-weight $\tilde{w}_t(k)$ to be $\tilde{w}_t(k)=\frac{w_t(k)}{\eta_k}$, and $\tilde{W}_t=\sum_k \tilde{w}_t(k)$. We are going to prove $\tilde{W}_t\le t(\log t +1)$. When $t=1$, only $\mA_1$ is active therefore $\tilde{W}_1=1\le 1(\log 1+1)$. Assume now the claim holds for any $\tau\le t$, we decompose $\tilde{W}_{t+1}$ as
\begin{align*}
    \tilde{W}_{t+1}&=\sum_{k} \tilde{w}_{t+1}(k)\\
    &=\sum_{k, 2^k|t+1} \tilde{w}_{t+1}(k)+\sum_{k, 2^k\nmid t+1} \tilde{w}_{t+1}(k)\\
    &\le \log(t+1)+1+\sum_{k, 2^k\nmid t+1} \tilde{w}_{t+1}(k),
\end{align*}
because there are at most $\log(t+1)+1$ number of different intervals in $S$ starting at time $t+1$, where each such interval has initial weight $\tilde{w}_{t+1}(k)=1$. Now according to the induction hypothesis,
\begin{align*}
    \sum_{k, 2^k\nmid t+1} \tilde{w}_{t+1}(k)&=\sum_k \tilde{w}_t(k)(1+\eta_k \tilde{r}_t(k))\\
    &=\tilde{W}_t+\sum_k\tilde{w}_t(k)\eta_k \tilde{r}_t(k)\\
    &\le t(\log t +1)+\sum_k w_t(k) \tilde{r}_t(k).
\end{align*}

We complete the argument by showing that $\sum_k w_t(k) \tilde{r}_t(k)\le 0$. By the definition of $\tilde{r}_t(k)$ and convexity of $\tilde{\ell}_t$, we have that
\begin{align*}
    \sum_k w_t(k) \tilde{r}_t(k)&=W_t\sum_k p(t)_k\left(\tilde{\ell}_t\left(\sum_k p(t)_k \mA_k(t)\right) -\tilde{\ell}_t(\mA_k(t))\right)\\
    % &=W_t\left(\hat{\ell}_t^{\top}\left(\sum_k p(t)_k v(t,k)\right) -\left(\sum_k p(t)_k v(t,k)\right)\right)=0.
    &\le W_t\sum_k p(t)_k\left(\sum_k p(t)_k \tilde{\ell}_t\left(\mA_k(t)\right) -\tilde{\ell}_t(\mA_k(t))\right)\\
    &=W_t\left(\sum_k p(t)_k \tilde{\ell}_t\left(\mA_k(t)\right)-\sum_k p(t)_k \tilde{\ell}_t\left(\mA_k(t)\right)\right)=0.
\end{align*}

Hence, using the inequality that $\log(1+x)\ge x-x^2$ for $x\ge -\frac{1}{2}$, we have that for the $k$ satisfying $2^k=|I|$,
\begin{align*}
    2\log t&\ge \log(\tilde{w}_t(k))=\sum_{t\in I}\log (1+\eta_k \tilde{r}_t(k))\\
    &\ge \sum_{t\in I}\eta_k \tilde{r}_t(k)-\sum_{t\in I}(\eta_k \tilde{r}_t(k))^2\\
    &\ge \eta_k \left(\sum_{t=j}^s \tilde{r}_t(j,k)-\eta_k I G^2D^2 \log^2 T\right),
\end{align*}
because $|\tilde{r}_t(k)|\le GD \log T$. Rearranging the above inequality and taking an expectation over the observations, we get the following bound
$$
\mathbb{E}\left[\sum_{t\in I}\tilde{r}_t(k)\right] \le \eta_k I G^2D^2 \log^2 T+\frac{2 \log T}{\eta_k}=O(GD\sqrt{I}\log^{1.5}T).
$$
Combining this with the previous regret bound for experts, we have that for any such interval $I\in S$, the overall regret can be bounded by $O\left(dGD\sqrt{I}\log^{1.5} T\right)$. Finally, apply the same argument as in A.2 of  \cite{daniely2015strongly} (also Lemma 7 in \cite{lu2022adaptive}).  This allows us to extend the regret bound over $I\in S$ to any interval $I$ at the cost of an additional $\sqrt{\log T}$ term, by observing that any interval can be written as the union of at most $\log T$ number of disjoint intervals in $S$ and using Cauchy-Schwarz; see Theorem 1 in \cite{daniely2015strongly}.
\end{proof}

\section{An Alternative Proof Strategy for BCO}
As we mentioned in the main-text, it's possible to further improve the number of queries in our algorithm from three to two, by combining it with the linear surrogate loss idea from \cite{zhao2021bandit}. We provide a brief algorithm description and proof here, based on the following observation. Let's denote $x_t$ as the point actually played by the algorithm and $g_t$ the sub-gradient of $\hat{\ell}_t$ at $x_t$, then for any interval $I$ and any expert $\A_k$, we have that
\begin{align*}
    \text{convexity:} &\sum_{t\in I} \hat{\ell}_t(x_t)-\hat{\ell}_t(x^{'*}_I)\le \sum_{t\in I}g_t^{\top}(x_t-x^{'*}_I),\\
    \text{decomposition:} &\sum_{t\in I}g_t^{\top}(x_t-x^{'*}_I)=\sum_{t\in I}g_t^{\top}(x_t-\A_k(t))+\sum_{t\in I}g_t^{\top}(\A_k(t)-x^{'*}_I).
\end{align*}
The above observation allows us to reduce to problem of minimizing the regret $\sum_{t\in I} \hat{\ell}_t(x_t)-\hat{\ell}_t(x^{'*}_I)$, to minimizing the tracking regret $\sum_{t\in I}g_t^{\top}(x_t-\A_k(t))$ and the expert regret $\sum_{t\in I}g_t^{\top}(\A_k(t)-x^{'*}_I)$ of an (adaptive) linear surrogate loss $h_t(x)=g_t^{\top}x$. Luckily, the linear loss is the same across all experts, thus we only need to make one gradient estimation of $g_t$ at each round.

The algorithm is similar to Algorithm \ref{alg bco}, with a few differences:
\begin{enumerate}
    \item in line 9, we don't sample $k_t$ anymore.
    \item in line 10, we make only one additional observation $\ell_t(x_t+\delta_t \mathbf{u})$ instead.
    \item in line 11, the gradient estimator is now constructed as $\hat{g}_t=\frac{d\mathbf{u}(\ell_t(x_t+\delta_t \mathbf{u})-\ell_t(x_t))}{\delta_t}$.
    \item in line 12, $\tilde{\ell}_t(\A_k(t))=\hat{g}_t^{\top} \A_k(t)$, $\tilde{\ell}_t(x_t)=\hat{g}_t^{\top} x_t$.
\end{enumerate}

Now we analyze the regret of the modified algorithm. The overall proof strategy is similar to that of Theorem \ref{thm: bco}. First we consider the regret of the pseudo loss $\hat{h}_t(x)=\hat{g}_t^{\top} x$. By the optimality of expert algorithms, we have that for some $k$ which corresponds to an expert $\A_k$ which optimizes $I$,
$$
\sum_{t\in I}\hat{g}_t^{\top}(\A_k(t)-x^*_I)=O(dGD\sqrt{I}),
$$
because $\|\hat{g}_t\|_2\le dG$ by definition. In addition, we have that
$$
\sum_{t\in I}\hat{g}_t^{\top}(x_t-\A_k(t))=\tilde{O}(GD\sqrt{I}),
$$
by the same argument on the tracking regret as in the proof of Theorem \ref{thm: bco}. Use the fact that $\hat{g}_t$ is an unbiased estimation of $g_t$, we reach the following upper bound on expected regret:
$$
\mathbb{E}\left[\sum_{t\in I} \hat{\ell}_t(x_t)-\hat{\ell}_t(x^{'*}_I)\right]\le \mathbb{E}\left[\sum_{t\in I}g_t^{\top}(x_t-x^{'*}_I)\right]=\mathbb{E}\left[\sum_{t\in I}\hat{g}_t^{\top}(x_t-x^{'*}_I)\right]=\tilde{O}(dGD\sqrt{I}).
$$

It's left to bound the estimation error between $\hat{\ell}_t$ and the true loss $\ell_t$. As already shown in the proof of Theorem \ref{thm: bco}, similarly we have that
$$
\sum_{t\in I} \E[\ell_t(x_t)]- \sum_{t\in I} \ell_t(x^*_I)\le \sum_{t\in I} \E[\hat{\ell}_t(x_t)]-\sum_{t\in I} \hat{\ell}_t(x^{'*}_I)+(2+\kappa)\delta DGT.
$$
Combining the two upper bounds, the expected regret of this algorithm is bounded by $\tilde{O}(dGD\sqrt{I}+\kappa\delta DGT)=\tilde{O}(dGD\sqrt{I})$, with only two queries per round. To sum up, this algorithm combines the (1) strongly adaptive regret framework from our algorithm \ref{alg bco}, (2) the linear surrogate loss idea from \cite{zhao2021bandit}. (1) improves the sub-optimal $\tilde{O}(\sqrt{T})$ adaptive regret of \cite{zhao2021bandit}, while (2) reduces the required number of queries by one.

\iffalse
\section{Additional Experiments}

Additional plots for other BBOB functions are included here. We note that the performance of StABL is relatively consistently better than Uniform, which is a reasonably competitive baseline in many more explorative settings. However, in some more exploitative settings, as expected Exp3 will perform better than the adaptive algorithm; however on average over explorative and exploitative settings, the adaptive strategy can adapt and performs as well as Uniform or Exp3 over generally all settings.

\begin{figure}[htb!]
\centering
\includegraphics[width=.82\linewidth]{neurips_adaptive.pdf}
\end{figure}

\begin{figure}[htb!]
\centering
\vskip-11in
\includegraphics[width=.82\linewidth]{neurips_adaptive.pdf}
\end{figure}

\begin{figure}[htb!]
\centering
\vskip-22in
\includegraphics[width=.82\linewidth]{neurips_adaptive.pdf}
\end{figure}
\fi

\section{Additional Experiments}

We provide additional experiments in the learning from expert advice setting. The experimental setup is the same, but with more arms and time steps. In the additional experiments, we take $N=300$, $T=65536$, and we run all algorithms 10 times. We plot the moving average with a window of 500.  

Our findings are largely consistent with the previous experiments. The left subfigure of Figure \ref{random_rewards_more} corresponds to the performance of EXP3, StABL, and StABL Naive on four intervals of equal length, where in each interval we have a different best arm. It is clear that StABL attains the best rewards among the three algorithms. EXP3 has good rewards in the first interval, but is not able to adapt when the environment changes. In the right subfigure, we plot the performance of algorithms when the best arm changes at time 7853, 13822, 25180, and 56621. In this experiment, EXP3 underperforms both StABL and StABL Single Scale, while StABL Single Scale also has difficulty adapting to the different interval lengths, especially at the beginning where the environment changes quickly.

\begin{figure}
\centering
\begin{subfigure}{0.5\textwidth}
  \centering
  \includegraphics[width=.95\linewidth]{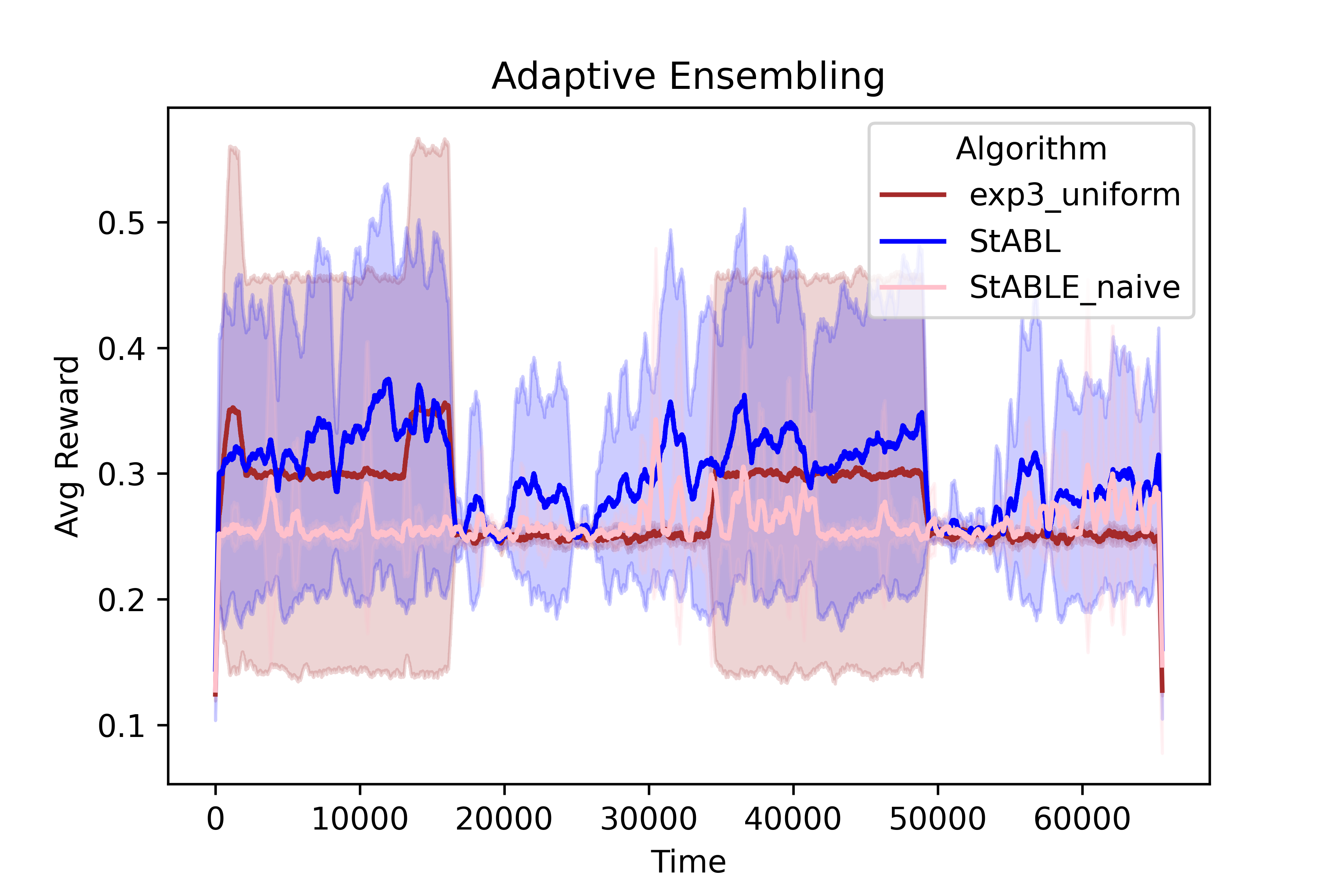}
  \label{fig:sub3}
\end{subfigure}%
\begin{subfigure}{0.5\textwidth}
  \centering
  \includegraphics[width=.95\linewidth]{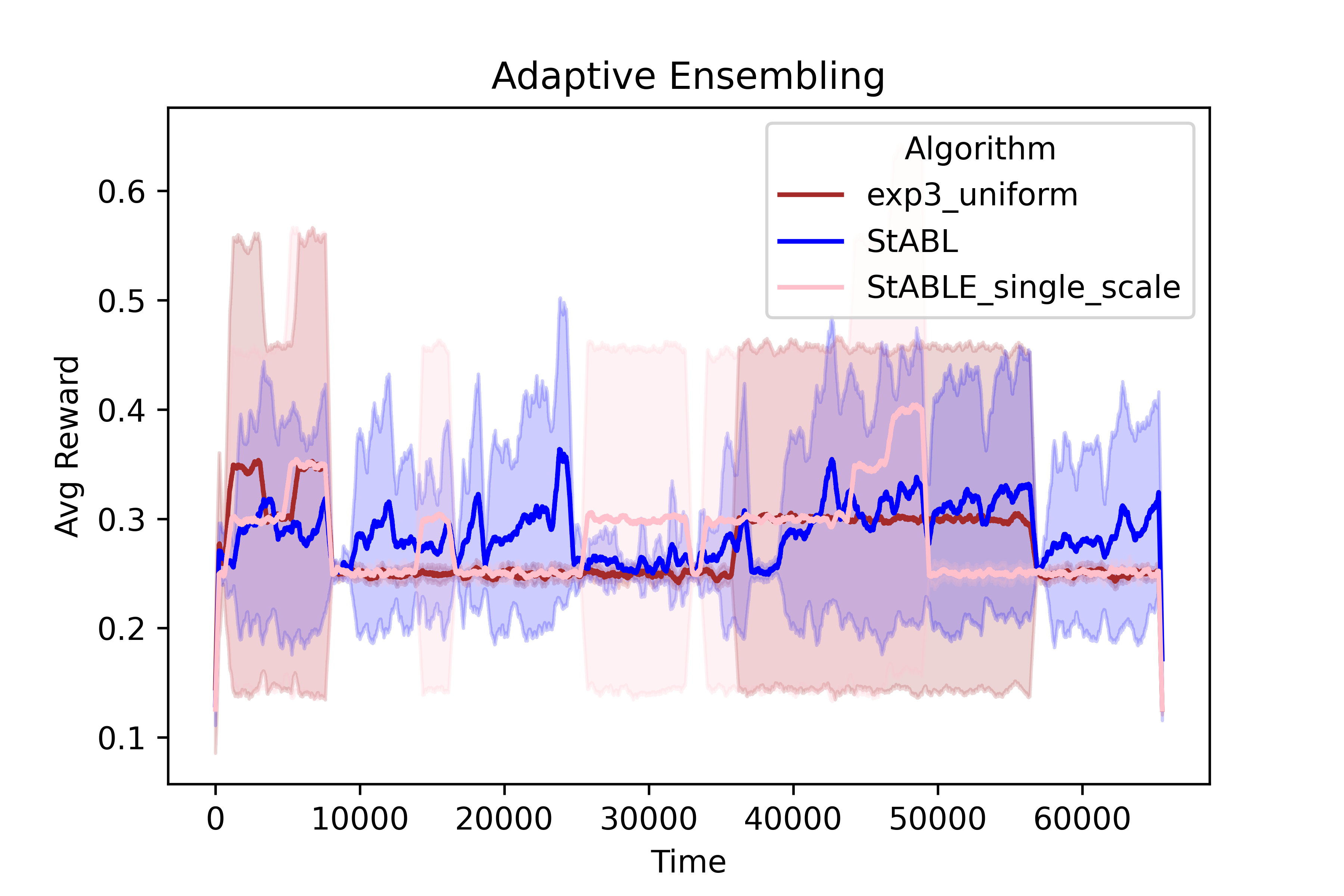}
  \label{fig:sub4}
\end{subfigure}
\caption{Further comparison plots of the algorithm rewards in the learning with expert advice setting.}
\label{random_rewards_more}
\end{figure}

% \fi

\end{document}